\documentclass{article}

\usepackage[final]{neurips_2021}

\usepackage[utf8]{inputenc} 
\usepackage[T1]{fontenc}    
\usepackage{hyperref}       
\usepackage{url}            
\usepackage{booktabs}       
\usepackage{amsfonts}       
\usepackage{nicefrac}       
\usepackage{microtype}      
\usepackage{xcolor}         

\usepackage{amsmath}
\usepackage{bm}
\usepackage{enumitem}
\usepackage{graphicx}

\usepackage{times}

\usepackage{soul}
\usepackage[small]{caption}

\title{SSMF: Shifting Seasonal Matrix Factorization}

\author{%
    
    Koki Kawabata \thanks{Equal contribution}\\
    SANKEN Osaka University\\
    \texttt{koki@sanken.osaka-u.ac.jp} \\
    \And
    Siddharth Bhatia \footnotemark[1] \\
    National University of Singapore \\
    \texttt{siddharth@comp.nus.edu.sg} \\
    \And
    Rui Liu \\
    National University of Singapore \\
    \texttt{xxliuruiabc@gmail.com} \\
    \And
    Mohit Wadhwa \\
    \texttt{mailmohitwadhwa@gmail.com} \\
    \And
    Bryan Hooi \\
    National University of Singapore \\
    \texttt{bhooi@comp.nus.edu.sg} \\

}

\usepackage{algorithm}
\usepackage{algorithmic}
\usepackage{color}
\usepackage{xspace}
\usepackage{bbding}
\usepackage{pifont}
\usepackage{wasysym}
\usepackage{amssymb}
\usepackage{wrapfig}
\def\equationautorefname~#1\null{Equation~(#1)\null}

\newtheorem{problem}{Problem}

\newtheorem{lemma}{Lemma}
\newtheorem{proof}{Proof}

\newcommand{\OK}{\checkmark}
\newcommand{\NO}{-}

\newcommand{\argmin}{\mathop{\rm arg~min}\limits}

\newcommand{\Nrow}{m}
\newcommand{\Ncol}{n}
\newcommand{\Npoint}{r}
\newcommand{\Nper}{s}  
\newcommand{\Nstate}{k}
\newcommand{\Nregime}{g}

\newcommand{\Queue}{\mathcal{X}}

\newcommand{\Xt}{\mathbf{X}(t)}
\newcommand{\X}{\mathbf{X}}
\newcommand{\Xhat}{\hat{\mathbf{X}}}

\newcommand{\z}{z}

\newcommand{\U}{\mathbf{U}}
\newcommand{\V}{\mathbf{V}}
\newcommand{\W}{\mathbf{W}}
\newcommand{\Wi}[1]{\mathbf{W}^{(#1)}}  
\newcommand{\Wfull}{\mathcal{W}}

\newcommand{\ub}{\mathbf{u}}  
\newcommand{\vb}{\mathbf{v}}
\newcommand{\ws}{\mathbf{w}}

\newcommand{\CostM}[1]{<#1>}
\newcommand{\CostC}[2]{<#1|#2>}
\newcommand{\CostT}[2]{<#1;#2>}
\newcommand{\cF}{c_F}

\newcommand{\lr}{\alpha}  

\newcommand{\method}{SSMF\xspace}
\newcommand{\smf}{SMF\xspace}
\newcommand{\trmf}{TRMF\xspace}
\newcommand{\ncp}{NCP\xspace}
\newcommand{\regimecast}{RegimeCast\xspace}

\newcommand{\NYtaxi}{NYC-YT\xspace}
\newcommand{\NYtaxiNrow}{265}
\newcommand{\NYtaxiNcol}{265}
\newcommand{\NYtaxiNtime}{4368}

\newcommand{\NYCbikeU}{NYC-CB\xspace}
\newcommand{\NYCbikeUNrow}{409}
\newcommand{\NYCbikeUNcol}{64}
\newcommand{\NYCbikeUNtime}{35785}

\newcommand{\Tycho}{DISEASE\xspace}
\newcommand{\TychoNrow}{57}
\newcommand{\TychoNcol}{36}
\newcommand{\TychoNtime}{3370}

\newcommand{\mypara}[1]{\noindent\textbf{#1}.}

\begin{document}

\maketitle

\begin{abstract}

Given taxi-ride counts information between departure and destination locations, how can we forecast their future demands? In general, given a data stream of events with seasonal patterns that innovate over time, how can we effectively and efficiently forecast future events? In this paper, we propose \textbf{S}hifting \textbf{S}easonal \textbf{M}atrix \textbf{F}actorization approach, namely SSMF, that can adaptively learn multiple seasonal patterns (called regimes), as well as switching between them. Our proposed method has the following properties: (a) it accurately forecasts future events by detecting regime shifts in seasonal patterns as the data stream evolves; (b) it works in an online setting, i.e., processes each observation in constant time and memory; (c) it effectively realizes regime shifts without human intervention by using a lossless data compression scheme. We demonstrate that our algorithm outperforms state-of-the-art baseline methods by accurately forecasting upcoming events on three real-world data streams.

\end{abstract}

\section{Introduction}
\label{sec:intro}

The ubiquity of multi-viewed data has increased the importance of advanced mining algorithms that aim at forecasting future events \citep{Bai_undated-kl,wu2019trend,deng2020dynamic},
such as forecasting taxi-ride counts
based on their departure and destination locations, and time information,
where the events can be represented as a stream of tuples
$(departure, destination, time; count)$.
Infection spread rates during epidemics
can also be represented in a similar form
as $(location, disease, time; count)$,
and the task is to forecast the infection count.
Given such a large data stream produced by user events, how can we effectively forecast future events? How can we extract meaningful patterns to improve forecasting accuracy? Intuitively, the problem we wish to solve is as follows:
\begin{problem}
    Given an event data stream that contains seasonal patterns, forecast the number of future events between each pair of entities in a streaming setting.
\end{problem}

As natural periodicity occurs in several types of data, one of the promising directions to understand the vital dynamics of a data stream is to consider seasonal phenomena.
However, high-dimensional data structures in sparse time-series pose many challenging tasks to uncover such valuable seasonal patterns \citep{chi2012tensors},
while modeling seasonality or periodic patterns has been
a well-studied topic in literature \citep{Takahashi2017-qb}.
Moreover, in a streaming setting, data dynamics vary over time,
and thus the models learned with historical data can be ineffective as we observe new data.
Although online learning schemes, such as \citep{schaul2013no, jothimurugesan2018variance, Yang2018-bo,nimishakavi2018inductive, Lu2019-mz}, provide solutions to overcome this concern, disregarding all past dynamics fails to capture re-occurring patterns in the subsequent processes.
On non-stationary data streams, such approaches no longer converge to the optimal parameters but capture average patterns for several phenomena, which causes less accurate forecasting.
Adequate methods should retain multiple patterns where each pattern is updated incrementally for prevailing tendencies and leveraged for forecasting.

In this paper, we focus on online matrix factorization for real-time forecasting with seasonality. Specifically, we propose a streaming method, Shifting Seasonal Matrix Factorization (\method), which allows the factorization to be aware of multiple smooth time-evolving patterns as well as abrupt innovations in seasonal patterns,
which are retained as \textit{regimes}.
In summary, our contributions are:
\begin{enumerate}
    \item \textbf{Accuracy:}
    Outperforms state-of-the-art baseline methods by accurately forecasting $50$-steps ahead of future events.
    \item \textbf{Scalability:}
    Scales linearly with the number of regimes, and uses constant time and memory for adaptively factorizing data streams.
    \item \textbf{Effectiveness:}
    Finds meaningful seasonal patterns (regimes) throughout the stream and divides the stream into segments based on incrementally selected patterns.  \method uses a lossless data compression scheme to determine the number of regimes. 
\end{enumerate}
\textbf{Reproducibility:}
    Our datasets and source code are publicly available at:
    \url{https://www.github.com/kokikwbt/ssmf}

\section{Related Work}
\label{sec:relatedwork}
\begin{table}[t]
    \centering
    \small
    \caption{Relative Advantages of \method}
    \label{tab:comparison}
    \begin{tabular}{r|cccc|c}\toprule
        & {SVD/NMF++}
        & {\regimecast}
        & {\trmf}
        & {\smf}
        & {\textbf{\method}} \\
        \midrule
        Sparse Data             &\OK    &\NO &\OK &\OK &\OK \\
        Regime Shifts           &\NO    &\OK &\NO &\NO &\OK \\
        Seasonal Pattern        &\NO    &\NO &\OK &\OK &\OK \\
        Drifting components     &some    &\NO &\NO &\OK &\OK \\
        \bottomrule
    \end{tabular}
    \normalsize
\end{table}

\autoref{tab:comparison} summarizes the relative advantages
of \method, compared with existing methods 
for matrix factorization and time series forecasting.
Our proposed method satisfies all the requirements.
The following two subsections provide the details of related work
on matrix/tensor factorization and time series modeling.

\subsection{Matrix Factorization}
The major goal of matrix/tensor factorization,
well known as SVD,
is to obtain lower-dimensional linear representation
that summarizes important patterns \citep{chi2012tensors},
and it has a wide range of applications
\citep{hoang2017gpop,yelundur2019detection,tivsljaric2020spatiotemporal}.
Non-negative constraint provides an ability
to handle sparse data and find interpretable factors
\citep{cichocki2009nonnegative}.
\citep{Takahashi2017-qb} adapted tensor factorization
to extract seasonal patterns in time series
as well as the optimal seasonal period
in an automatic manner.
\citep{yu2016temporal} formulated
the temporal regularized matrix factorization (TRMF),
which is designed for time series prediction.
\citep{Song2017-hs,Najafi2019-uc}
addressed stream tensor completion but
didn't extract seasonal patterns.
Recently, \smf\citep{hooi2019smf} proposed
an online method that incorporates seasonal factors into
the non-negative matrix factorization framework,
and outperforms CPHW \citep{dunlavy2011temporal},
which is an offline matrix factorization method.
However, \smf misses regime shifts
which should be modeled individually.

\subsection{Time Series Analysis}
Statistical models,
such as autoregression (AR) and Kalman filter (KF),
learn temporal patterns for forecasting,
thus a lot of their extensions have been proposed
\citep{cai2015facets,dabrowski2018state}.
\regimecast \citep{Matsubara2016-fx}
focuses on regime detection to obtain compact factors
from data streams effectively
but it does not consider updating detected regimes.
Also, note that they do not consider mining sparse data
in the non-linear dynamical systems,
while handling missing values has a significant role
in time series analysis
\citep{yi2016st,abedin2019sparsesense}.
Hidden Markov models are also powerful tools
for change detection in time series
\citep{
mongillo2008online,%
montanez2015inertial,%
Pierson2018-nq,%
kawabata2019automatic},
but these models represent only short temporal dependency
and thus are ineffective for forecasting tasks.
Detecting regime shifts in non-stationary data
can be categorized in Concept Drift,
and its theory and recent contributions
are well-summarized in \citep{Lu2019-mz}.
In more detail, we consider
incremental drift and reoccurring concepts simultaneously.
\citep{wen2019robuststl}
tried to decompose
seasonality and trends in time series
on the sparse regularization framework
to handle fluctuation and abrupt changes in seasonal patterns,
but this approach does not focus on stream mining
\citep{krempl2014open}.

\section{Proposed Model}
\label{sec:model}

In this section, we propose a model for shifting seasonal matrix factorization.
A tensor we consider consists of a timestamped series of
$(\Nrow\times\Ncol)$ matrices $\X(1),\dots,\X(\Npoint)$,
which can be sparse, until the current time point $\Npoint$.
We incrementally observe a new matrix $\X(\Npoint + 1)$
and $\Npoint$ evolves ($\Npoint=\Npoint + 1$).
Our goal is to forecast $\X(t)$ where $\Npoint < t$
by uncovering important factors in the flow of data,
whose characteristics can change over time.
As we discussed the effectiveness
of handling seasonal patterns, 
we incorporate seasonal factors
into a switching model
that can adaptively recognize recent patterns in forecasting.

\subsection{Modeling Shifting Seasonal Patterns}
\autoref{fig:model} shows an overview of our model.
In matrix factorization,
we assume that each attribute,
e.g., a start station where we have a taxi ride,
has $\Nstate$ latent communities/activities,
which is smaller than $\Nrow$ and $\Ncol$,
e.g., the number of stations.
Considering such latent seasonal activities,
an input data stream is decomposed into the three factors:
$\U$ with shape $(\Nrow\times\Nstate)$ and
$\V$ with shape $(\Ncol\times\Nstate)$
for the row and columns of data $\Xt$, respectively,
and seasonal patterns $\W$ with period $\Nper$,
i.e., $(\Nper\times\Nstate)$,
where $\Nstate$ is the number of components to decompose.
The $i$-th row vector $\ws_i=\{w_1,\dots,w_{\Nstate}\}$
of $\W$ corresponds to the $i$-th season.
Letting $\W_i$ be a $(\Nstate\times\Nstate)$ diagonal matrix
(zeros on non-diagonal entries) with diagonal elements
as $\ws_i$, $\Xt$ is approximated as follows.
%
\begin{align}
    \Xt \approx \U\W_i\V^T.
\end{align}
More importantly, we extend $\W$
by stacking additional $\W$
to capture multiple discrete seasonal patterns
as \textit{regimes}, which is shown as the red tensor
in \autoref{fig:model}.
Letting $\Nregime$ be the number of regimes,
our model employs a seasonality tensor
$\Wfull=\{\W^{(1)},\dots,\W^{(\Nregime)}\}$,
where each slice $\W^{(i)}$ captures
the $i$-th seasonal pattern individually.
By allowing all these components to vary over time,
$\X(t)$ can be represented as follows.
\begin{align}
    \Xt \approx \U(t)\W_i^{(\z)}(t)\V(t)^T,
\end{align}
where $\z\in\{1,\dots,\Nregime\}$
shows an optimal regime index to represent $\X(t)$
in the $i$-th season.
That is, our approach adaptively selects $\ws_{i}^{(z)}$,
as shown in the red-shaded vector in the figure,
with regard to the current season and regime.
%
Finally, the full parameter set we want to estimate is as follows.

\begin{problem}[Shifting seasonal matrix factorization: \method]
\textbf{Given:} a tensor stream $\X(1),$ $\dots,$ $\X(\Npoint)$
that evolves at each time point,
\textbf{Maintain:}
community factors: $\U(t),\V(t)$,
seasonal factors: $\Wfull(t)$,
and the number of regime $\Nregime$ in a streaming fashion.
\end{problem}

\begin{figure}
    \centering
    \includegraphics[width=0.75\linewidth]{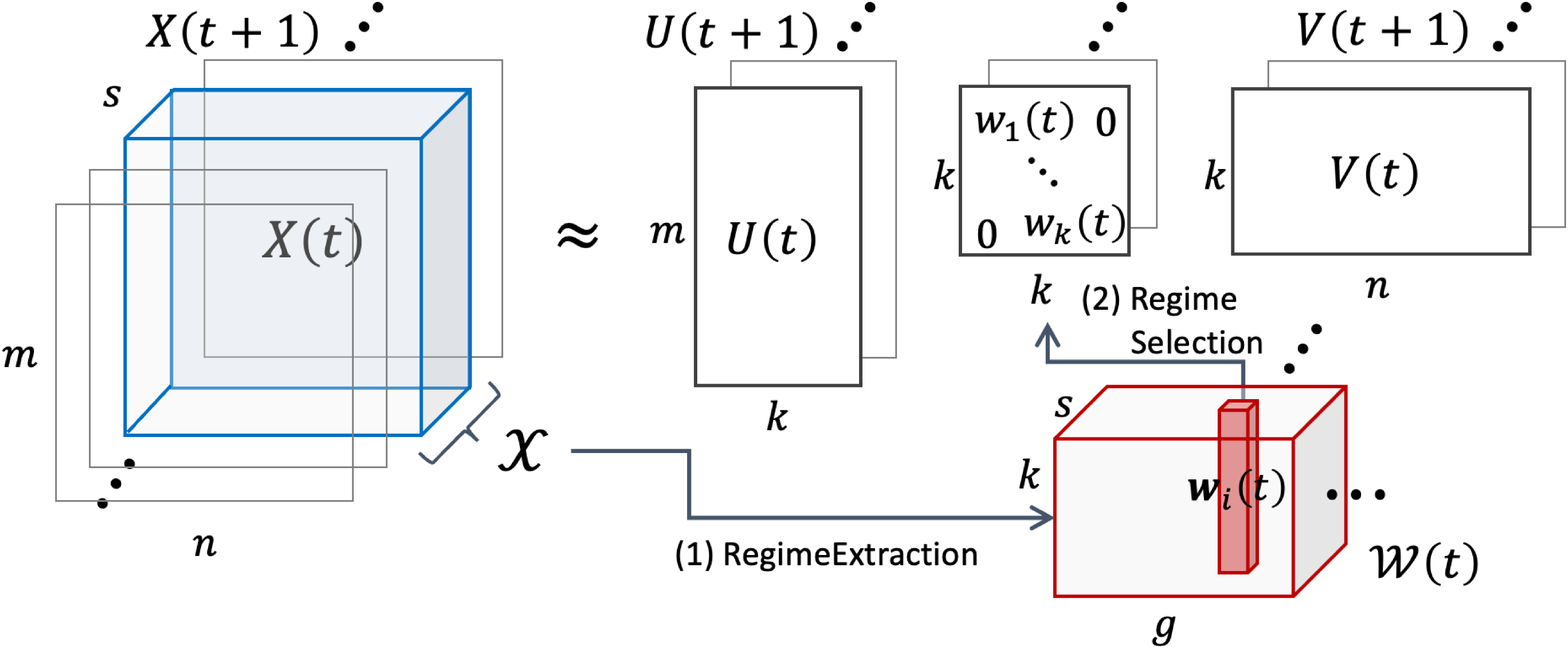}
    \caption{An overview of our model:
    it smoothly updates all the components,
    $\U(t),\V(t)$, and $\Wfull(t)$,
    to continuously describe data stream $\X(t)$.
    If required, (1) increases the number of regimes, $\Nregime$,
    by extracting a new seasonal pattern
    in the last season $\Queue=\{\X(t-s+1),\dots,\X(t)\}$.
    Then, (2) selects the best $\ws_{i}^{(z)}(t)$ in regimes
    to describe $\X(t)$.
    }
    \label{fig:model}
\end{figure}


\subsection{Streaming Lossless Data Compression}
In the \method model, deciding the number of regimes, $\Nregime$,
is a crucial problem for accurate modeling
because it should change as we observe a new pattern to shift.
It is also unreliable to obtain any error thresholds
to decide the time when employing a new regime
before online processing.
%
Instead, we want the model to employ a new regime incrementally
so that it can keep concise yet effective
to summarize patterns.
The minimum description length (MDL) principle \citep{grunwald2007minimum}
realizes this approach, which is defined as follows.
\begin{align}
    \underbrace{
        \CostT{\X}{\U,\V,\Wfull}
    }_{\textrm{Total\ MDL\ cost}}
    \ =\
    \underbrace{
        \CostM{\U} + \CostM{\V} + \CostM{\Wfull}
    }_{\textrm{Model\ description\ cost}}
    + \underbrace{
        \CostC{\X}{\U,\V,\Wfull}
    }_{\textrm{Data\ encoding\ cost}}, 
\end{align}
where $\CostT{\X}{\U,\V,\Wfull}$ is the total MDL cost of $\X$
when given a set of factors, $\U,\V$ and $\Wfull$.

Model description cost, $\CostM{\cdot}$, measures
the complexity of a given model, based on the following equations.
\begin{align}
    \CostM{\U} &= |\U|\cdot(\log(\Nrow)+\log(\Nstate)+\cF),\nonumber\\ 
    \CostM{\V} &= |\V|\cdot(\log(\Ncol)+\log(\Nstate)+\cF),\nonumber\\
    \CostM{\Wfull} &= 
        |\Wfull|\cdot(
            \log(\Nregime)+\log(\Nper)+\log(\Nstate)+\cF
        ),\nonumber
\end{align}
where $|\cdot|$ shows the number of nonzero elements
in a given matrix/tensor, $\cF$ is the float point cost\footnote{%
We used $\cF=32$ bits.}.

Data encoding cost, $\CostC{\X}{\cdot}$,
measures how well a given model compress original data,
for which the Huffman coding scheme \citep{adler01towards}
assigns a number of bits to each element in $\X$.
Letting $\hat x\in\hat\X$ be a reconstructed data,
the data encoding cost is represented by
the negative log-likelihood under a Gaussian distribution
with mean $\mu$ and variance $\sigma^2$ over errors,
as follows.
\begin{align}
    \CostC{\X}{\U,\V,\Wfull}
        = \sum_{x\in\X}-\log_2 P_{\mu,\sigma}(x - \hat x).
\end{align}
While we cannot compute the cost with all historical data
in a streaming setting, a straightforward approach obtains the best number of regimes so that the total description cost is minimized.
To compute the total cost incrementally and individually when provide with new data, we decompose $\CostM{\Wfull}$ approximately into:
\begin{align}
    \CostM{\Wfull}\ \approx
    \sum_{i=1}^{\Nregime}\CostM{\Wi{i}}\ =
    \sum_{i=1}^{\Nregime}|\Wi{i}|\cdot(\log(\Nper)+\log(\Nstate)+\cF).
\end{align}
When a new matrix $\X$ arrives, 
we decide whether the current $\Wfull$
should increase the number $\Nregime$.
by comparing
the total costs for existing regimes
with the total costs for estimated regime.

\begin{align}\label{eqn:costWonline}
    \W_{best}
        &=\argmin_{\W'\in\{\Wi{i},\dots,\Wi{\Nregime+1}\}}
            \CostT{\X}{\U,\V,\W'},\\
        &=\argmin_{\W'\in\{\Wi{i},\dots,\Wi{\Nregime+1}\}}
            \CostM{\U}+\CostM{\V}+\CostM{\W'} + \CostC{\X}{\U,\V,\W'}. \nonumber 
\end{align}
If $\Wi{\Nregime+1}$ gives the minimum total cost of $\X$,
then it belongs to the full $\Wfull$ for subsequent modeling.
In other words, the size of $\Wfull$ is encouraged to be small
unless we find significant changes in seasonal patterns.
%

\section{Proposed Algorithm}
\label{sec:algorithm}

We propose an online algorithm for SSMF
on large data streams containing seasonal patterns. Algorithm~\ref{alg:ssmf} summarizes the overall procedure of our proposed method: it first detects the best regime for the most recent season through a two-step process, namely, RegimeSelection and RegimeExtraction.
Then it smoothly updates the entire factors based on the selected regime with non-negative constraints.
When deciding the number and assignment of regimes, the algorithm evaluates the cost, \autoref{eqn:costWonline}, for the most recent season, to ensure the smoothness of consecutive regime-switching during a season. Thus, the algorithm keeps a sub-tensor $\Queue$,
which consists of the matrices from time point $t-s+1$ to $t$, as a queue while receiving a new $\X(t)$. Given a $\Queue$, the best regime at time point $t$ is obtained as follows.

\begin{itemize}
\item 
RegimeSelection, which aims to choose the best regime $\W_{RS}$
that minimizes \autoref{eqn:costWonline}
in $\{\Wi{1},\dots,\Wi{\Nregime}\}$,
using the two previous factors $\U(t-1)$ and $\V(t-1)$.
We compute the total cost, $C_{RS}$
by using $\U(t-1),\V(t-1)$ and $\W_{RS}$ for $\Queue$.
\item
RegimeExtraction, which obtains a new regime $\W_{RE}$,
focuses on current seasonality, by applying the gradient descent method to
$\U(t-1),\V(t-1)$ and $\W_{RS}$ over the whole $\Queue$.
The initialization with $\W_{RS}$ prevents
$\W_{RE}$ from overfitting to $\Queue$
because temporal data, $\Queue$, can be sparse.
It then computes $C_{RE}$ with $\U(t-1),\V(t-1)$ and $\W_{RE}$ for $\Queue$.
\end{itemize}
%
If $C_{RS}$ is less than $C_{RE}$,
the best regime index $z$ 
becomes one of $\{1,\dots,\Nregime\}$.
Otherwise, $z=\Nregime+1$,
$\Wfull(\Nregime\times\Nper\times\Nstate)$
is extended to
$\Wfull((\Nregime+1)\times\Nper\times\Nstate)$,
and then $\Nregime$ is incremented.

\subsection{Regime-Aware Gradient Descent Method}

Given a new observation $\X(t)$,
we consider obtaining the current factors,
$\U(t),\V(t),\Wfull(t)$
via the previous ones $\U(t-1),\V(t-1),\Wfull(t-1)$
by minimizing error with respect to $\X(t)$.
This allows the model to fit current dynamics smoothly
while considering the past patterns.
Thus, our algorithm performs the gradient descent method
with a small learning rate $\lr>0$.
The gradient step keeps the error with respect to $\Xt$ low.
Letting $\z$ be the current regime index,
and $\hat\X(t)$ be the predicted matrix of $\X(t)$
with the previous parameters:
\begin{align}
    \hat\X(t)=\U(t-1)\W^{(\z)}(t-\Nper)\V(t-1)^T,
\end{align}
the gradient update to $\ub_i(t)$ and $\vb_i(t)$ is given by:
\begin{align}
    \ub_i(t)\!&\leftarrow\!\ub_i(t-1)
    \!+\! \lr(\X(t)-\hat\X(t))\vb_i(t-1)w_i^{(\z)}(t-\Nper),\\
    \vb_i(t)\!&\leftarrow\!\vb_i(t-1)
    \!+\! \lr(\X(t)-\hat\X(t))^T\ub_i(t-1)w_i^{(\z)}(t-\Nper).
\end{align}
The updated $\ub_i(t)$ and $\vb_i(t)$
are projected such that
\begin{align}
    \ub_i(t)\leftarrow\max(0, \ub_i(t));
    \vb_i(t)\leftarrow\max(0, \vb_i(t))
\end{align}
to keep the parameters non-negative.
After all, the algorithm normalizes $\ub_i(t)$ and $\vb_i(t)$
by dividing by their norms $\|\ub_i(t)\|$ and $\|\vb_i(t)\|$,
respectively:
\begin{align}
    \ub_i(t)&\leftarrow\ub_i(t) / \|\ub_i(t)\|;~~
    \vb_i(t) \leftarrow\vb_i(t) / \|\vb_i(t)\|,
\end{align}
and then $w_i^{(\z)}$ multiplied by the two norms
keeps the normalization constraint:
\begin{align}
    w_i^{(\z)} &\leftarrow
    w_i^{(\z)}(t-\Nper)\cdot\|\ub_i(t)\|\cdot\|\vb_i(t)\|,
\end{align}
which captures the seasonal adjustments for the $i$-th component.

%

\begin{algorithm}[t]
    \small
    \caption{\method}
    \label{alg:ssmf}
\begin{algorithmic}[1]
    \REQUIRE
        (a) New observation $\X(t)$ and 
        previous factors $\U(t-1),\!\V(t-1),\!\Wfull(t-1)$, 
    \ENSURE
        Updated factors: $\U(t),\V(t),\Wfull(t)$
    \STATE \textcolor{blue}{/* Update the queue to keep the most recent season */}
    \STATE $\Queue\leftarrow\{\X(t-\Nper+1),\dots,\X(t)\}$;
    \STATE \textcolor{blue}{/* (1) RegimeSelection */}
    \STATE
        $\W_{RS}\leftarrow\argmin_{\W'\in\{\Wi{1},\dots,\Wi{\Nregime}\}}
            \CostT{\Queue}{\U(t-1),\V(t-1),\W'}$;
    \STATE \textcolor{blue}{/* (2) RegimeExtraction */}
    \STATE 
        $\W_{RE}\leftarrow\argmin_{\W'}
            \CostC{\Queue}{\U(t-1),\V(t-1),\W'}$;
    \STATE
        $C_{RS}=\CostT{\Queue}{\U(t-1),\V(t-1),\W_{RS}}$;
        $C_{RE}=\CostT{\Queue}{\U(t-1),\V(t-1),\W_{RE}}$;
    \IF {$C_{RS}$ \textbf{is less than} $C_{RE}$}
        \STATE $\Wi{z}=\textrm{diag}(\W_{RS}(t))$;
    \ELSE
        \STATE $\Wi{z}=\textrm{diag}(\W_{RE}(t))$;~
               $\Wfull(t)=\Wfull(t-1)\cup\W_{RE}$;~
               $\Nregime=\Nregime+1$;
    \ENDIF
    \STATE \textcolor{blue}{/* Perform gradient updates */}
    \STATE $\hat{\X}(t)\leftarrow \U(t-1)\W^{(\z)}(t-\Nper)\V(t-1)^T$;
    \STATE $\U(t)\leftarrow\U(t-1)
            +\lr(\X(t) - \hat{\X}(t))\V(t-1)\Wi{z}(t-\Nper)$;
    \STATE $\V(t)\leftarrow\V(t-1)
            +\lr(\X(t) - \hat{\X}(t))^T\U(t-1)\Wi{z}(t-\Nper)$;

    \FOR {$i=1:\Nstate$}
        \STATE \textcolor{blue}{/* Project onto non-negative constraint */}
        \STATE
            $\ub_i(t)\leftarrow\max(0, \ub_i(t))$;~~
            $\vb_i(t)\leftarrow\max(0, \vb_i(t))$
        \STATE \textcolor{blue}{/* Renormalize all the factors */}
        \STATE
            $w_i^{(z)}\leftarrow
             w_i^{(z)}(t-\Nper)\cdot\|\ub_i(t)\|\cdot\|\vb_i(t)\|$;
        \STATE
            $\ub_i(t)\leftarrow\ub_i(t) / \|\ub_i(t)\|$;~~
            $\vb_i(t) \leftarrow\vb_i(t) / \|\vb_i(t)\|$;
    \ENDFOR
    \RETURN $\U(t),\V(t),\Wfull(t)$
\end{algorithmic}
\normalsize
\end{algorithm}

\subsection{Forecasting}

To forecast $t$-steps ahead values
for a given $t > r$,
we use the latest factors,
$\U(\Npoint)$, $\V(\Npoint)$, and $\W_{t_s}^{(z)}(r)$, where
$t_s$ is the most recent season observed before time $\Npoint$,
which corresponds to the season at time $t$.
Namely, $t_s = \Npoint - (\Npoint - t \mod \Nper)$.
By choosing the best regime $z$,
which gives the minimum description cost
in the last one season,
the predicted matrix for time $t$ is thus computed as follows.
\begin{align}
    \Xhat(t)=\U(r)\W_{t_s}^{(i)}(r)\V(r).
\end{align}

\subsection{Theoretical Analysis}

We show the time and memory complexity of \method.
The key advantage is that both of the complexities
are constant with regard to a given tensor length, even if it evolves without bound.

\begin{lemma}\label{lem:time}
    The time complexity of \method is
    $O(\Nper\Nregime\Nstate\Nrow\Ncol + \Nper\Nstate^2(\Nrow+\Ncol))$
    per time point.
\end{lemma}

\begin{proof}
    The algorithm first computes
    the model description cost
    between $\X(t)$ and $\hat\X(t)$ using $\Nregime$ regimes in
    $O(\Nregime\Nstate\Nrow\Ncol)$,
    which is repeated over $t=t-\Nper+1,\dots,\Npoint$.
    To obtain a new regime,
    it performs the gradient descent update to $\Queue$,
    which takes
    $O(\Nper(\Nstate\Nrow\Ncol+\Nstate^2(\Nrow+\Ncol)))$.
    Then, for a selected regimes,
    it performs the gradient descent update, which requires
    $O(\Nstate\Nrow\Ncol+\Nstate^2(\Nrow+\Ncol))$.
    The projection to non-negative values
    and normalization require $O(\Nrow+\Ncol)$ for $\Nstate$ loops,
    thus $O(\Nstate(\Nrow+\Ncol))$ in total.
    Overall, the time complexity of \method is
    $O(\Nper\Nregime\Nstate\Nrow\Ncol + \Nper\Nstate^2(\Nrow+\Ncol))$
    per time point.
\end{proof}

\begin{lemma}\label{lem:memory}
    The memory complexity of \method is
    $O(\Nstate(\Nrow + \Ncol + \Nregime\Nper) + |\Queue|)$
    per time point.
\end{lemma}

\begin{proof}
\method retains the two consecutive factors,
$\U$ and $\V$ with shape $(\Nrow\times\Nstate)$
and $(\Ncol\times\Nstate)$, respectively,
which requires $O(\Nstate(\Nrow + \Ncol))$.
For seasonal factors in $\Nregime$ regimes,
it maintains a tensor $\Wfull$ with shape
$(\Nregime\times\Nper\times\Nstate)$.
It also needs a space to retain a sparse tensor, $|\Queue|$,
which is the number of observed elements.
Therefore, the total memory space required for \method is
$O(\Nstate(\Nrow + \Ncol + \Nregime\Nper) + |\Queue|)$
per time point.
\end{proof}

\section{Experiments}
\label{sec:experiments}
\begin{table}[t]
    \centering
    \caption{Dataset Description}
    \label{table:dataset}
    \begin{tabular}{lrrrrrrr}
    \toprule
    Name & Rows & Columns & Duration
    & Sparsity & Frequency & $s$ & $r_{test}$\\
    \midrule
    \NYtaxi    & \NYtaxiNrow    & \NYtaxiNcol & \NYtaxiNtime
    & $98.3$\% & Hourly         & $168$ & $500$ \\
    \NYCbikeU  & \NYCbikeUNrow  & \NYCbikeUNcol & \NYCbikeUNtime
    & $96.2$\%   & Hourly         & $168$ & $500$ \\
    \Tycho     & \TychoNrow     & \TychoNcol & \TychoNtime
    & $92.4$\% & Weekly         & $52$ & $200$ \\
    \bottomrule
    \end{tabular}
\end{table}

In this section, we show how our proposed method
works on real tensor streams.
The experiments were designed to answer the following questions.
 
\begin{itemize}
    \item[\textbf{Q1.}] \textbf{Accuracy:}
    How accurately does it predict future events?  
    \item[\textbf{Q2.}] \textbf{Scalability:}
    How does it scale with input tensor length?
    \item[\textbf{Q3.}] \textbf{Effectiveness:}
    How well does it extract meaningful latent dynamics/components?
\end{itemize}
We implemented our algorithm in Python (ver. 3.7.4)
and all the experiments were conducted on an Intel Xeon W-$2123$
$3.6$GHz quad-core CPU with $128$GB of memory and running Linux.
Table \ref{table:dataset} shows the datasets
used to evaluate our method:

\mypara{\NYtaxi}\footnote{\url{https://www1.nyc.gov/site/tlc/about/tlc-trip-record-data.page}}
NYC Yellow Taxi trip records during the term from 2020-01-01 to 2020-06-30
in the pairs of $\NYtaxiNrow$ pick-up and drop-off locations.


\mypara{\NYCbikeU}\footnote{\url{https://s3.amazonaws.com/tripdata/index.html}}
NYC CitiBike ride trips from 2017-03-01 to 2021-03-01.
We only considered stations that had at-least $50$k rentals during the term, and extracted $409$ stations.
We then categorized users by age from $17$ to $80$
and obtained the stream of ride counts at the $409$ stations.

\mypara{\Tycho}\footnote{\url{https://www.tycho.pitt.edu/data/}}
A stream of the weekly number of infections
by $36$ kinds of diseases at the $57$ states/regions in the US,
which were collected from 1950-01-01 to 2014-07-27.


\subsection{Accuracy}
\label{subsec:accuracy}
We first describe a comparison in terms of
forecasting using Root Mean Square Error (RMSE).
Since \citep{hooi2019smf} outperforms several baselines
including CPHW and TSVDCWT,
we omit the baselines from our experiments.
Our baseline methods
include (1) matrix/tensor factorization approaches: \ncp \citep{xu2013block}:
a non-negative CP decomposition,
optimized by a block coordinate descent method
for sparse data.
(2) \smf \citep{hooi2019smf}:
an online algorithm for seasonal matrix factorization,
which manages a single seasonal factor.
(3) \trmf \citep{yu2016temporal}:
a method to take into account temporal and seasonal patterns
in matrix factorization.

For each method, the initialization process
was conducted on the first three seasons of an input tensor.
For \trmf,
we searched for the best three regularization coefficients,
$\lambda_I$, $\lambda_{AR}$, $\lambda_{Lag}$,
in $\lambda=\{0.0001,0.001,0.01,0.1,1,10\}$.
Auto-coefficients of \trmf are set for at time points
$r-1,\dots,r-\lceil\Nper/2\rceil$
as well as at three time points
$r -\Nper,r - (\Nper+1),r -(\Nper+2)$
to consider seasonality.
For \method and \smf, we determined a learning rate $\lr$
in $\lr=\{0.1,0.2,0.3,0.4\}$ by cross-validation in each training data.
The number of components $\Nstate$ was set to $15$ among all these methods
for a fair comparison.
After the initialization,
each algorithm takes the first $r_{train}$ time steps,
and forecasts the next $r_{test}$ steps,
as defined in \autoref{table:dataset}.
This procedure was repeated for every $r_{test}$ steps,
for example, $r_{train}=\{500, 1000, 1500, \dots\}$ in \NYtaxi dataset.


\autoref{fig:rmse} shows the results of RMSE
on all three datasets,
where \method consistently outperforms its baselines.
We observe that our method shows
a better improvement on (b) \NYCbikeU than (a) \NYtaxi,
whereas the two datasets have similar characteristics
because detecting regime shifts is more effective in forecasting
in a longer data duration.
\Tycho dataset contains repetitive regime shifts
caused by pandemics and developments of vaccines,
therefore, \method performs forecasting effectively.

\begin{figure}[t]
    \centering
    \includegraphics[width=0.85\columnwidth]{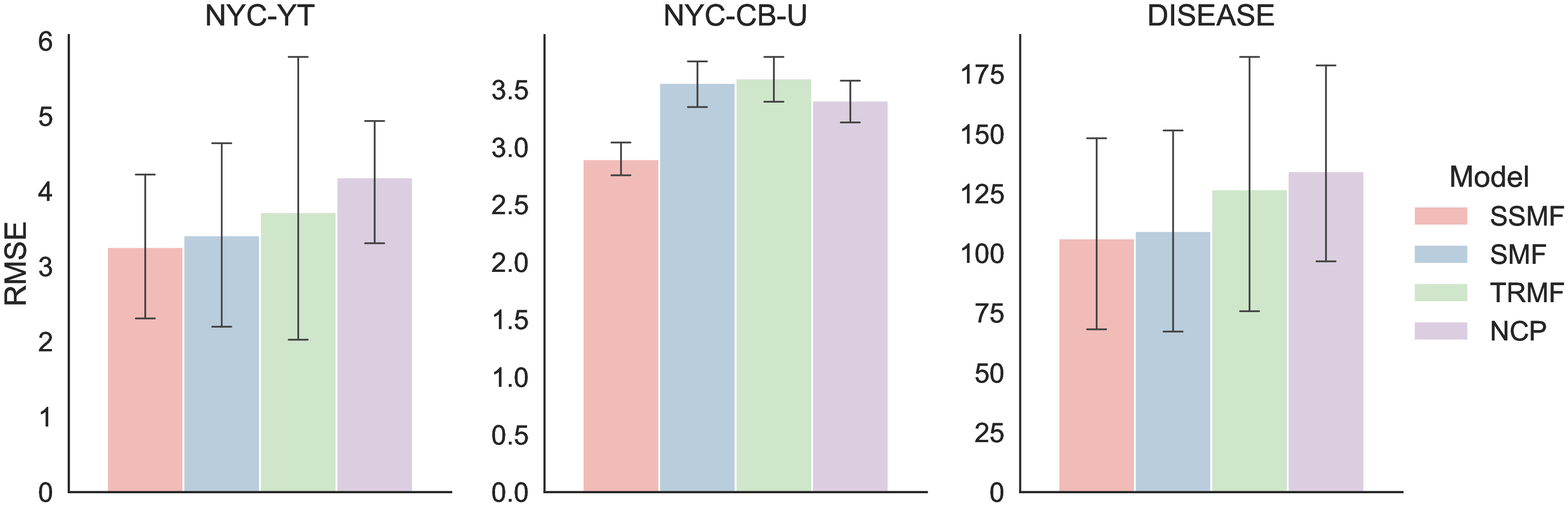}\\
    \caption{Forecasting performance in terms of RMSE:
        \method consistently outperforms the baselines thereby showing that regime shift detection adapts the model to recent patterns
        more quickly and effectively than the baselines.}
    \label{fig:rmse}
    \vspace{1em}
    \begin{tabular}{ccc}
        \includegraphics[width=0.3\columnwidth]{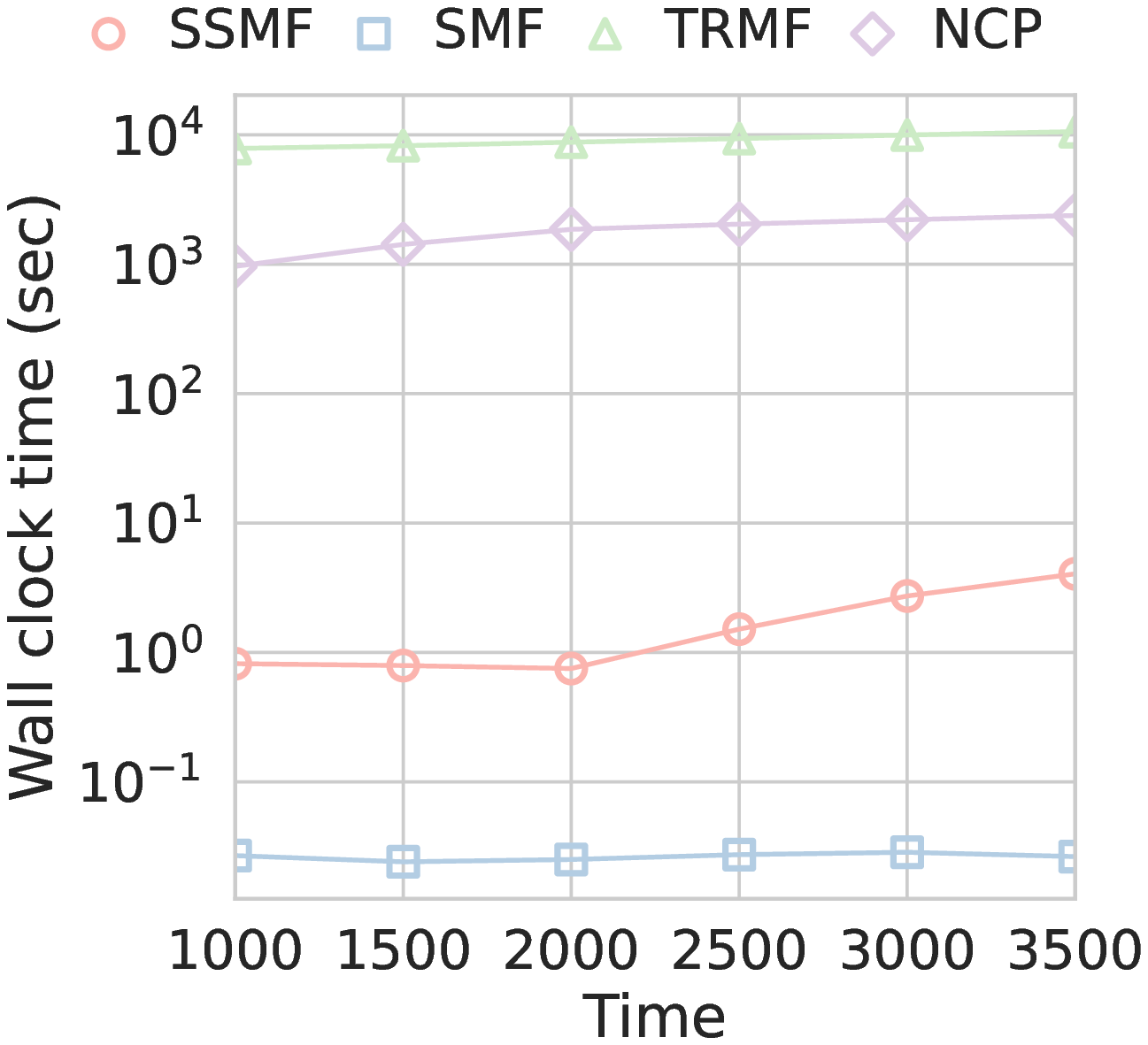}&
        \includegraphics[width=0.3\columnwidth]{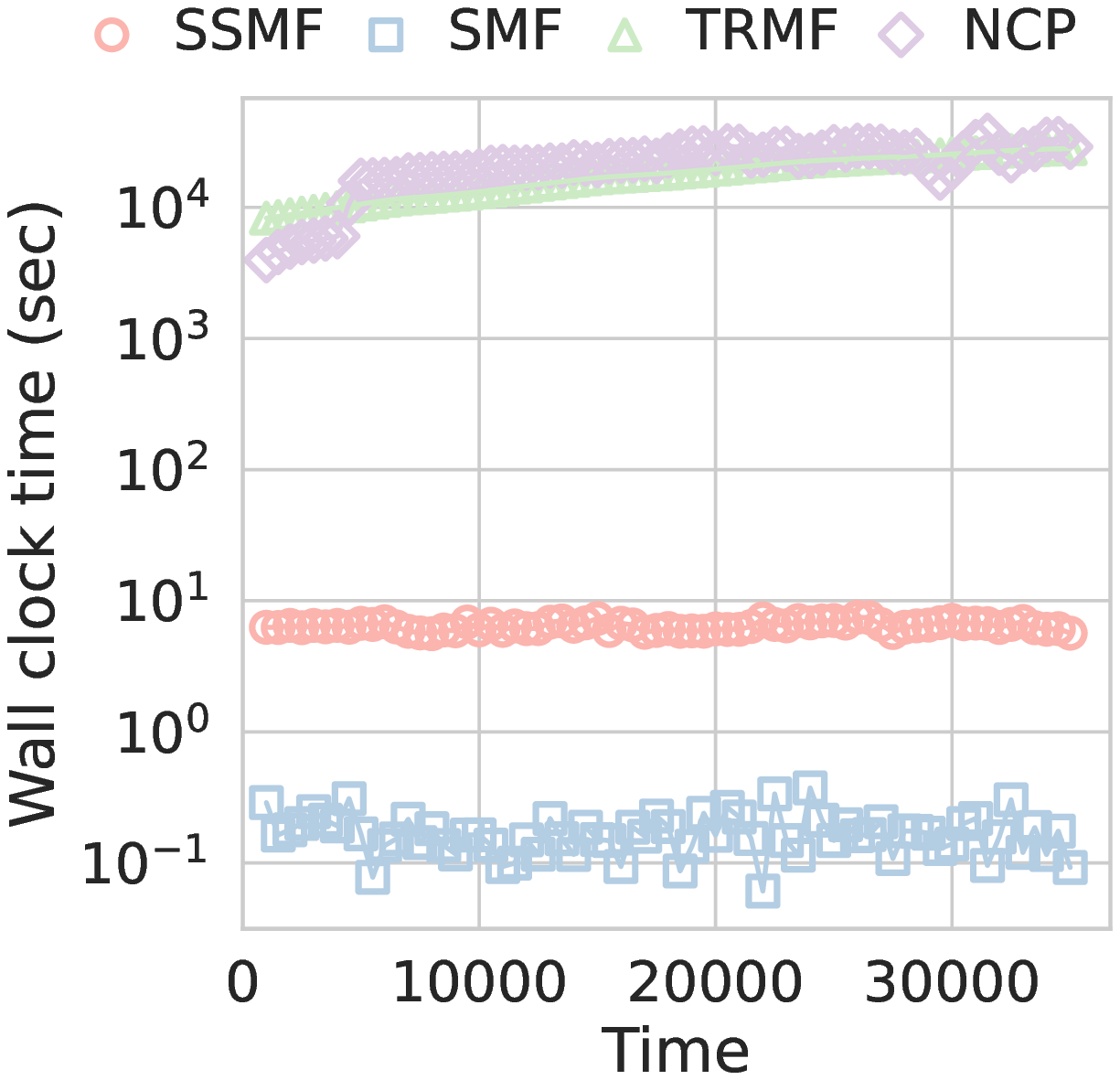}&
        \includegraphics[width=0.3\columnwidth]{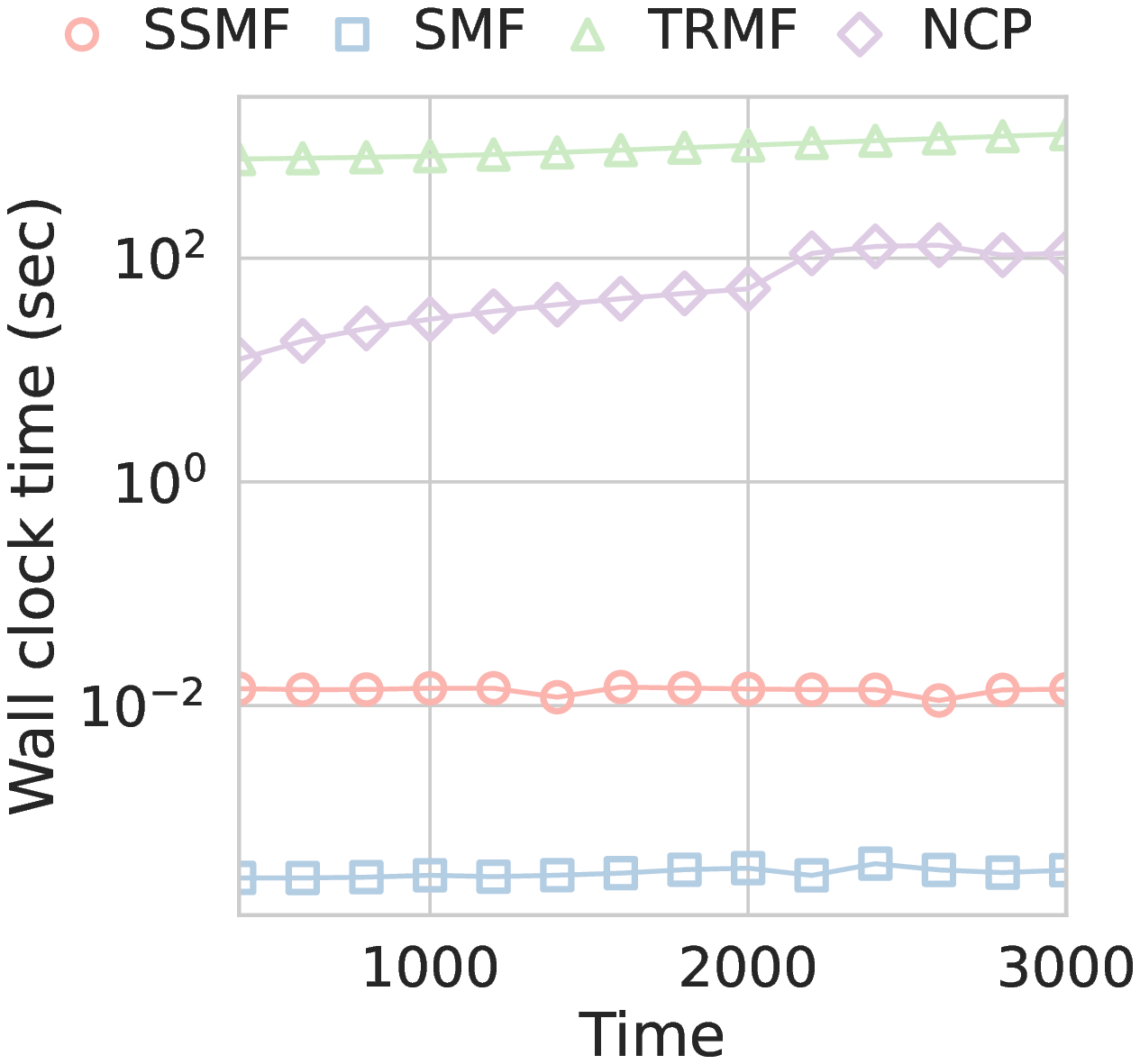}\\
        (a) \NYtaxi&
        (b) \NYCbikeU&
        (c) \Tycho
    \end{tabular}
    \caption{
        Wall clock time vs. data stream length:
        \method efficiently monitors data streams to find regimes.
    }
    \label{fig:walltime}
\end{figure}

\subsection{Scalability}
We discuss the performance of \method
in terms of computational time on the three datasets.
\autoref{fig:walltime} shows the wall clock time vs.
data stream length, where the experimental setting
is same as of \autoref{subsec:accuracy}.
The figures show that our proposed method
is scalable for mining and forecasting large sparse streams.
Although the scalability of \method
is linear along the number of regimes,
the growth rate of its running time
is greatly less than that of the other offline baselines.
This is because the proposed method can
summarize distinct seasonal patterns as regimes effectively.
\smf is the fastest algorithm but
it does not consider regime shifts in online processing.
As described in the accuracy comparison section,
our method is more effective when data streams
have significant pattern shifts,
which cannot be captured only with smooth updates
on a single, fixed component.

\begin{figure}[t]
    \centering
    \includegraphics[width=\columnwidth]{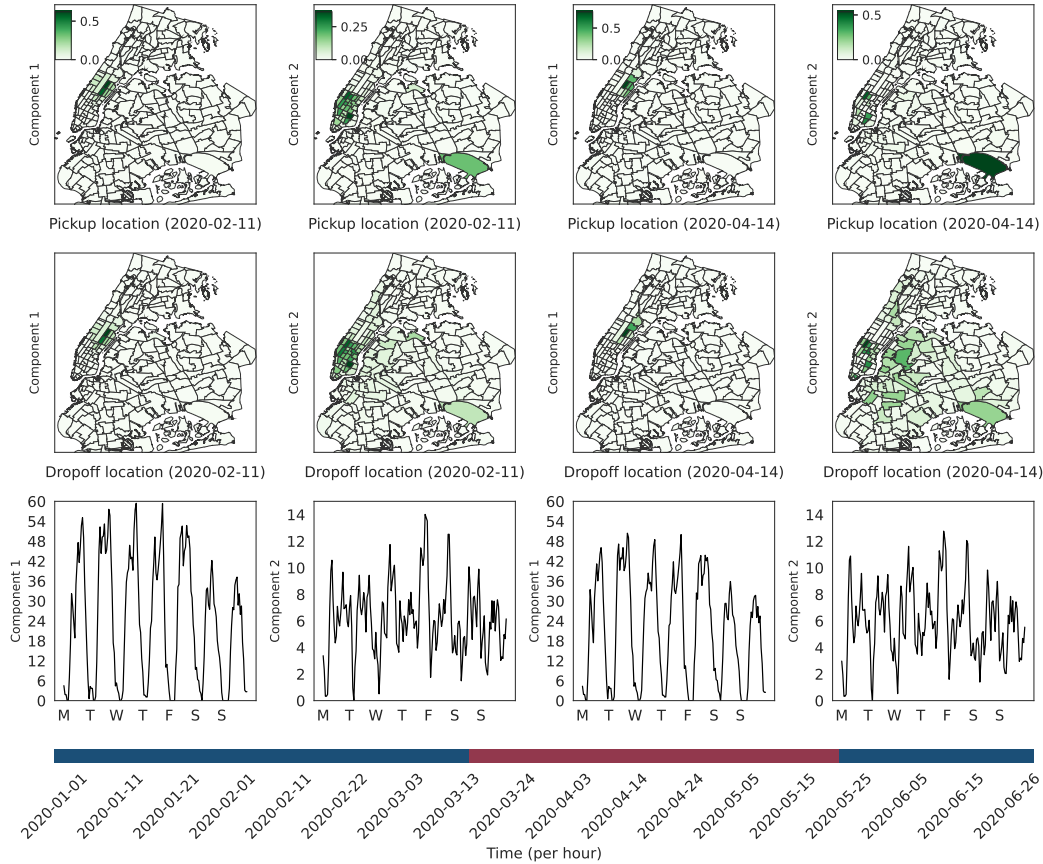}
    \caption{
        \method can automatically detect regime shifts caused by a pandemic
        based on multiple seasonal factors:
        Component 1,
        corresponding roughly to entertainment,
        decays in the pandemic,
        while component 2,
        representing trips between Manhattan and Kennedy Airport,
        experiences a shift toward trips
        from the airport to wide areas of New York City.
    }
    \label{fig:regimes}
\end{figure}

\subsection{Real-world Effectiveness}

In this section, we show an example of
the real-world effectiveness of our approach using the \NYtaxi dataset.
\autoref{fig:regimes} shows the overview of
the two representative regimes detected by \method,
where we selected two of the fifteen components to describe.
The bottom bar represents the timeline of regime detection.
\method kept using the initial regime during the blue term,
while it generated new regimes during the red term.
The three rows of plots are the visualization of $\W$, $\U$, $\V$, where the first two columns correspond to the blue regimes, and the next two to the red.

In mid March 2020, New York City's governor issued stay-at-home orders, and \method detected the event as the regime shifts
to adapt to the changed demands of taxi rides.
Separating regimes and learning individual patterns
help us to interpret the results as follows:
Component 1 of the blue regime shows a seasonal pattern that follows a weekday/weekend pattern.
According to the above two maps, 
more taxi rides are concentrated around Manhattan
before the pandemic.
In the red regime, however,
the strength of the seasonal component became
about $10$\% weaker than the one before.
Accordingly, the frequent pickup and dropoff became smaller.
This area has a large number of bars and restaurants,
suggesting that this component corresponds roughly to entertainment-related trips, whose demands are strongly affected by the pandemic.
Component 2, in contrast, is concentrated between Manhattan and Kennedy Airport, suggesting that this component captures trips to and fro the airport. Under the pandemic situation, this component shifted toward trips from the airport to destinations scattered around New York City. This observation could suggest that customers arrived at the airport,
and then used a taxi to avoid public transit. 
After this phase, our method switched to the first regime again, returning to the previous seasonal pattern. The demand for taxi rides began to recover, which roughly aligns with reopening in New York City, which began in early June.
Consequently, regime shifts allow us to distinguish
significant changes in seasonal behavior 
as well as contextual information between attributes.
In online learning, the mechanism also supports
for quick adaptation and updates,
which contributes to improving forecasting accuracy.

\section{Conclusion}
\label{sec:conclusion}

In this paper, we proposed Shifting Seasonal Matrix Factorization (\method) method, for effectively and efficiently forecasting future activities in data streams. The model employs multiple distinct seasonal factors to summarize temporal characteristics as regimes, in which the number of regimes is adaptively increased based on a lossless data encoding scheme.
We empirically demonstrate that on three real-world data streams,
our forecasting method \method has the following advantages over the baseline methods:
(a) By using the proposed switching mechanism in an online manner,
\method \textit{accurately} forecasts future activities.
(b) \textit{Scales} linearly by learning
matrix factorization in constant time and memory.
(c)  Helps understand data streams by \textit{effectively} extracting complex time-varying patterns.


\bibliographystyle{unsrtnat}
{\small \bibliography{references}}

\end{document}